\documentclass[letterpaper]{article} % DO NOT CHANGE THIS
\usepackage[]{aaai24}  % DO NOT CHANGE THIS
\usepackage{times}  % DO NOT CHANGE THIS
\usepackage{helvet}  % DO NOT CHANGE THIS
\usepackage{courier}  % DO NOT CHANGE THIS
\usepackage[hyphens]{url}  % DO NOT CHANGE THIS
\usepackage{graphicx} % DO NOT CHANGE THIS
\usepackage{natbib}  % DO NOT CHANGE THIS AND DO NOT ADD ANY OPTIONS TO IT
\usepackage{caption} % DO NOT CHANGE THIS AND DO NOT ADD ANY OPTIONS TO IT
\usepackage{algorithm}
\usepackage{newfloat}
\usepackage{listings}
\DeclareCaptionStyle{ruled}{labelfont=normalfont,labelsep=colon,strut=off} % DO NOT CHANGE THIS
\floatstyle{ruled}
\newfloat{listing}{tb}{lst}{}
\floatname{listing}{Listing}
\newcommand{\primarykeywords}{
(2) Learning;
(8) Knowledge Representation/Engineering
}

\usepackage[switch, modulo]{lineno}
\usepackage{amsthm}
\usepackage{amssymb}
\usepackage{mathtools}
\usepackage[noend]{algpseudocode}
\usepackage{xcolor}

\algnewcommand{\LineComment}[1]{\State \(\#\) #1}
\newcommand{\mytilde}{\raise.17ex\hbox{$\scriptstyle\mathtt{\sim}$}}
\newcommand{\Dp}{D_p}
\newcommand{\De}{D_e}
\newcommand{\Hp}{\H^a_p}
\newcommand{\He}{\H^a_e}

\newtheorem{thm}{Theorem}

\newtheorem{lemma}[thm]{Lemma}
\newtheorem{corollary}[thm]{Corollary}
\newtheorem{definition}{Definition}

\newcommand{\checkit}[1]{}

\newcommand{\tuple}[1]{\langle #1 \rangle}

\newcommand{\A}{\mathcal{A}}

 \renewcommand{\H}{\mathcal{H}}
 
 \renewcommand{\L}{\mathcal{L}}
\newcommand{\M}{\mathcal{M}}

 \newcommand{\T}{\mathcal{T}}
\newcommand{\U}{\mathcal{U}} 
\newcommand{\V}{\mathcal{V}}

\newcommand{\pre}{\mathsf{pre}}
\newcommand{\eff}{\mathsf{eff}}
\newcommand{\Pre}{\mathsf{Pre}}
\newcommand{\Eff}{\mathsf{Eff}}

\newcommand{\prop}[1]{{\cal P}_{A}}

\usepackage{xspace}
\newcommand{\block}[0]{\textsc{Blocks}\xspace}
\newcommand{\mic}[0]{\textsc{Miconic}\xspace}
\newcommand{\driver}[0]{\textsc{Driverlog}\xspace}
\newcommand{\sate}[0]{\textsc{Satellite}\xspace}

\title{Action Model Learning with Guarantees}
\title{Action Model Learning with Guarantees}
\author {
    Diego Aineto,
    Enrico Scala
}
\affiliations {
    Department of Information Engineering, University of Brescia, Italy\\
    \{diego.ainetogarcia, enrico.scala\}@unibs.it
}

\begin{document}

\maketitle

\begin{abstract}
This paper studies the problem of action model learning with full observability. 
Following the learning by search paradigm by Mitchell, we develop a theory for action model learning based on version spaces that interprets the task as search for hypothesis that are consistent with the learning examples.
Our theoretical findings are instantiated in an online algorithm that maintains a compact representation of all solutions of the problem. 
Among these range of solutions, we bring attention to actions models approximating the actual transition system from below (sound models) and from above (complete models). We show how to manipulate the output of our learning algorithm to build deterministic and non-deterministic formulations of the sound and complete models and prove that, given enough examples, both formulations converge into the very same true model. Our experiments reveal their usefulness over a range of planning domains.
\end{abstract}

\section{Introduction}

% Action model learning is the task of finding an action model formulation that best explains some observation of an agent acting in the environment. The engineering of action models can be very complicated and error prone. Making action model learning automatic holds the promise of lighting this task and at the same time provide the user with another level of autonomy that is provably consistent with what observed in the environment.

% Action model learning is the task of computing an approximation of a domain's dynamics from demonstrations. The engineering of action models is complicated and error prone, constituting one of the main bottlenecks in the application of model-based reasoning \cite{kambhampati:2007:model-lite}. Automating this process holds the promise of enabling all the AI planning machinery over a provably consistent model of the domain.

The engineering of action models is complicated and error prone, constituting one of the main bottlenecks in the application of model-based reasoning \cite{kambhampati:2007:model-lite}. Automating this process holds the promise of enabling the AI planning machinery \cite{ghallab:2004:planning} over a provably consistent model of the domain. Action model learning tackles this problem by computing an approximation of a domain's dynamics from demonstrations. 

%% This is what is done
Most of the research in action model learning of the last two decades has been focused on learning under partial observability, investigating the application of different techniques with the aim of either improving the expressiveness of the learnt models or handling more incomplete and noisy demonstrations. Against this trend, \emph{Safe Action Model (SAM) Learning} \cite{Stern:2017:SAM,Juba:2021:LiftedSAM,Juba:2022:SSAM,Mordoch:2023:NSAM} is a family of algorithms that takes a step back to study the fully observable setting from a theoretical standpoint that puts the emphasis on the properties of the learnt model. In particular, SAM focuses on learning \emph{safe} models, i.e., those with which an agent can safely execute actions all the way to the goal.

%%THis is what we do
This paper deepens this theoretical first principled investigation through the lens of version spaces \cite{mitchell:1982:generalization}.
% We focus our attention on classical planning models and develop a framework to learn action preconditions and effects by maintaining a version space containing only hypotheses consistent with the learning examples. 
We focus our attention on classical planning models and develop a framework to learn action preconditions and effects by maintaining a version space of all hypothesis consistent with the demonstrations.
% We develop a framework to learn action preconditions and effects by maintaining a version space of all hypotheses consistent with the demonstrations. 
The computed version space provides an efficient representation of all solutions to the action model learning problem. Among these solutions, those at the boundaries have special properties. On one end there is pessimism and, on the other, optimism. The pessimistic form leads to construct \emph{sound} models, i.e., those never allowing the agent to take a wrong step as per the \emph{safe} property studied by SAM.
The optimistic form leads to \emph{complete} models, i.e., those with which an agent can speculate about the existence of a plan.
% Sound models implies the existence of a plan for the original model, but generate potentially many false negatives. On the other hand, complete models do not in general produce valid plans, but guarantee the existence of a plan if one exists in the original model. Our framework aims at getting the best of both worlds.
The sound model generates plans that are guaranteed to work with the true model, but will often discard valid plans. On the other hand, complete models do not in general produce valid plans, but guarantee the existence of a plan if one exists for the true model. Our framework aims at getting the best of both worlds, by showing that, much as sound models lead to deterministic planning formulations, complete models can be captured through the extra expressiveness of non-deterministic ones.

% by showing how deterministic and non-deterministic formulations can be used to capture these two models
% by showing that, much as sound models can be captured by deterministic planning formulations, complete models can be captured using a non-deterministic planning formulation.
% by showing how to the SAM deterministic planning formulation that , we can add a non-deterministic formulation

% While all hypotheses in the version space are consistent, those at the boundaries have special properties. On one side there is specificity and, on the other, generality. The most specific form leads to construct sound models, i.e., those with which an agent can safely execute actions all the way to the goal. The most general form leads to complete models, i.e., those with which an agent can speculate about the existence of a plan. Sound models implies the existence of a plan for the original model, but generate potentially many false negatives. On the other hand, complete models do not in general produce valid plans, but guarantee the existence of a plan if one exists in the original model. Our framework aims at getting the best of both worlds.

%%These are our results
The main contribution of our work is theoretical. Indeed, our investigation precisely establishes rules that heavily exploit the structure of the hypothesis space in order to learn all the solutions of an action model learning problem. Such rules materialize into an online algorithm that outputs a compact representation of the set of solutions. Then, we show how to manipulate this representation to extrapolate sound and complete action model formulations that, given enough demonstrations, converge at the very same true model.
Finally, we conduct an experimental evaluation over a range of domains with the purpose of understanding the usefulness of the proposed framework. Our findings demonstrate that the adoption of a sound and a complete model provides the agent with better reasoning capabilities earlier in the learning process. This is due to the fact that complete models can exploit negative demonstrations, too.

The paper is organized as follows: We start off with background material on action model learning and version spaces. Then we delve into building a precise mapping between these two worlds (Section \ref{sec:mapping}) outlining several theoretical results. Section \ref{sec:sound_complete} shows how to leverage this mapping to build sound and complete action models, which are then practically evaluated in Section \ref{sec:evaluation}. We conclude with related work and discussion (sections \ref{sec:related} and \ref{sec:discussion}).

\section{Preliminaries}\label{sec:preliminaries}

This section presents the basic notions around action model learning and version spaces.

\subsection{Action Model Learning}

An action model is a description of the capabilities of some agent, system or environment. In this work, we focus on learning deterministic action models with conjunctive preconditions \cite{mcdermott:1998:pddl}, as defined below.

\begin{definition}[Action Model]\label{def:action_model}
    An action model is a tuple $M = \tuple{F,A,\pre,\eff}$ where:
    \begin{itemize}
        \item $F = \{f_1, \ldots, f_n\}$ is a finite set of Boolean state variables called \emph{fluents}. A positive (resp. negative) \emph{literal} is $l = f$ (resp. $l = \neg f)$ and its completement is $\Bar{l} = \neg l$. We denote the set of all literals by $L$.
        % A valuation of a fact is a pair $(f,b) \in F \times \{0, 1\}$ known as \emph{literal} and we denote the set of all literals by $L$. 
        \item $A$ is a finite set of labels called \emph{actions}.
        \item $\pre : A \rightarrow 2^L$ defines the \emph{precondition} $\pre(a) \subseteq L$ for all $a \in A$.
        \item $\eff : A \rightarrow 2^L$ defines the \emph{effect} $\eff(a) \subseteq L$ for all $a \in A$.
    \end{itemize}
\end{definition}

Action models are often represented in lifted manner, by parameterising actions and fluents over a set of \emph{objects}. We adopt a ground represented for ease of presentation, but all our results are directly applicable to the lifted case. An action model succinctly represents a transition system where a state $s$ is an assignments over $F$, represented by a subset of $L$ without conflicting values. We denote by $S$ the state space induced by $F$.
% whose states are full assignments over the fluents $F$ and commonly represented by subsets of $L$ without conflicting values. 
% Formally, an action model $M = \tuple{F,A,\pre,\eff}$ induces the transition system $\T_M = \tuple{S, A, T}$ where the \emph{state space} $S \subseteq 2^{L}$ contains all possible assignments over the facts $F$, the set of \emph{labels} is $A$, and $T \subset S \times A \times S$ is the \emph{transition relation} defined as
Formally, an action model $M = \tuple{F,A,\pre,\eff}$ induces the transition system $\T_M = \{\tuple{s,a,s'} \in S \times A \times S \mid \pre(a) \subseteq s \land s' = (s \setminus \overline{\eff(a)}) \cup \eff(a)\}$
where $\overline{\eff(a)} = \{\overline{l} \mid l \in \eff(a)\}$.

% \da{I'm thinking of simplifying notation by defining a transition system directly as a relation. Meaning $\T_M \subset S \times A \times S$ instead of the tuple $\T_M = \tuple{S,A,T}$ that has inside the relation $T$. That would give us easier access to the transitions of $M$}

% . $T$ contains all transitions $\tuple{s,a,s'}$ such that $\pre(a) \subseteq s$ and $s' = s \setminus \eff^-(a) \cup \eff(a)^+$ where $\eff(a)^+$ and $\eff(a)^-$ are the subsets of $\eff(a)$ containing only assignments to $\mathtt{true}$ and $\mathtt{false}$, respectively. \da{add something about conjunctive preconditions and effects and deterministic transitions}.

Action models are widely used in AI planning to formulate reachability problems over the induced transition system. A \emph{classical planning problem} $P = \tuple{M, s_0, G}$ is defined by combining an action model $M = \tuple{F,A,\pre,\eff}$ with an \emph{initial state} $s_0 \in S$ and \emph{goal condition} $G \subseteq L$. A solution for $P$ is a sequence of actions $\pi = (a_1, \ldots a_n)$ known as \emph{plan} and its \emph{execution} in $s_0$ yields an interleaved sequence $\langle s_0, a_1, s_1, a_2, s_2, \ldots, a_n, s_n \rangle$ that alternates actions and states iteratively reached by applying the action one after the other. A plan is a valid solution if every \emph{transition} $\tuple{s_i, a_{i+1}, s_{i+1}}$ belongs to the transition system $\T_M$ and $G \subseteq s_n$. We denote the set of solution plans for $P$ by $\Pi(P)$.

Action model learning is about computing the action model of an agent from \emph{demonstrations} of its capabilities. Hereinafter, we denote by $\A$ the \emph{true action model} of the agent and assume that it complies with Definition $\ref{def:action_model}$. 
% We refer to the input data of the learning task as \emph{demonstrations}. 
Demonstrations are collected from executions of $\A$, e.g., a plan or random walk, and represented similarly to transitions.

\begin{definition}[Demonstration]
    A demonstration is a triple $d = \tuple{s,a,s'}$ consisting of a pre-state $s \in S$, an action $a \in A$, and a post-state $s' \in S \cup \{\bot\}$.
\end{definition}

In this work, we consider \emph{positive demonstrations}, those transitions of $\T_\A$, and \emph{negative demonstrations}, those representing the failure of executing action $a$ in state $s$ which we indicate by using a $\bot$ post-state.

An action model learning problem takes as input a set of fluents $F$, a set of actions $A$ and a set of demonstrations $D$. The aim of action model learning 
% is to compute the preconditions and the effects of actions in $A$ and definable using $F$ from the demonstrations in $D$.
is to find an action model that is \emph{consistent} with all the demonstrations in $D$.
It is worth noting that the space of models that can be synthesised given $F$ and $A$ is finite, i.e, $\M = \{\tuple{F,A,\pre,\eff} \mid \forall a \in A : \pre(a) \in 2^L \land \eff(s) \in 2^L\}$.

% $ \M = \{\tuple{F,A,\pre,\eff} \mid \pre \in (2^L)^{|A|} \land \eff \in (2^L)^{|A|}\}$.

% set $\M$ of possible action models is finite; there is a finite number of models that can be synthesised given $F$ and $A$, i.e., $ \M = \{\tuple{F,A,\pre,\eff} \mid \pre \in (2^L)^{|A|} \land \eff \in (2^L)^{|A|}\}$.

\begin{definition}[Action Model Learning Problem]\label{def:aml_problem}
An action model learning problem is a tuple $\Lambda = \tuple{F, A, D}$ where $F$ is a set of fluents, $A$ is a set of actions, and $D$ is a set of demonstrations. A solution for $\Lambda$ is an action model $M = \tuple{F,A,\pre,\eff}$ such that:
\begin{enumerate}
    \item for all positive demonstrations $\tuple{s,a,s'} \in D$, it holds that $\pre(a) \subseteq s$ and $s' = (s \setminus \overline{\eff(a)}) \cup \eff(a)$;
    \item for all negative demonstrations $\tuple{s,a,\bot} \in D$, it holds that $\pre(a) \not\subseteq s$.
\end{enumerate}
We denote by $\M_D$ the subset of the model space $\M$ that satisfies (1) and (2), i.e., the set of solutions of $\Lambda$.
\end{definition}

Hereinafter, we use $\M_D$ and \emph{consistent action models} interchangeably, depending on the nature of the text, to refer to the set of solutions of the action model learning problem.

% Under this definition, action model learning $\Lambda = \tuple{F, A, D}$ is about computing an action model synthesized over actions $A$ and fluents $F$ that is consistent with demonstrations $D$, and the set of solutions of $\Lambda$ are all consistent models $\M_D$.

% Under this definition, action model learning is about learning the transition relation of a transition system whose state space and actions are known.
% We also follow an inductive interpretation of learning so that the precondition and effects learnt comply with the positive examples and not with the negative examples.

\subsection{Version Spaces}

We adopt the notation and definitions introduced in later extensions of the version spaces framework \cite{lau:2003:algebra}. 

\begin{definition}[Hypothesis and Hypothesis Spaces]
    A \emph{hypothesis} is a function $h : I \rightarrow O$. A \emph{hypothesis space} $\H$ is a set of functions with the same domain and range.
\end{definition}

\begin{definition}[Learning Example]
    A learning example $\epsilon$ is a pair $(i,o) \in I \times O$. A hypothesis $h$ is consistent with a learning example $\epsilon = (i,o)$ if and only if $h(i) = o$.
\end{definition}

\begin{definition}[Version Space]
    Given a hypothesis space $\H$ and a set of learning examples $E$, the \emph{version space} $\V_{\H,E}$ is the subset of $\H$ that is consistent with all examples in $E$. We will often omit the subscripts if the hypothesis space and learning set are clear from the context.
\end{definition}

Let $\leq$ be a partial order relation between elements in $\H$, a version space $\V$ can be efficiently represented in terms of its least upper bound, called $\U$ boundary, and its greatest lower bound, called $\L$ boundary, relative to $\leq$ \cite{lau:2000:algebra}. The consistent hypotheses are those that belong to the boundaries or lie between them in the partial order. Formally, $\V_{\H,E} = \{ h \in \H \mid \exists (h_\L,h_\U) \in \L_{\H,E} \times \U_{\H,E} : h_\L \leq h \leq h_\U\}$. \emph{Version Space Learning} \cite{mitchell:1982:generalization} is an online algorithm that maintains a version space by updating its boundaries each time a new learning example is observed. The update function must ensure that everything within the new boundaries is consistent and everything outside, inconsistent.

\begin{figure}
    \centering
    \includegraphics[width=0.9\linewidth]{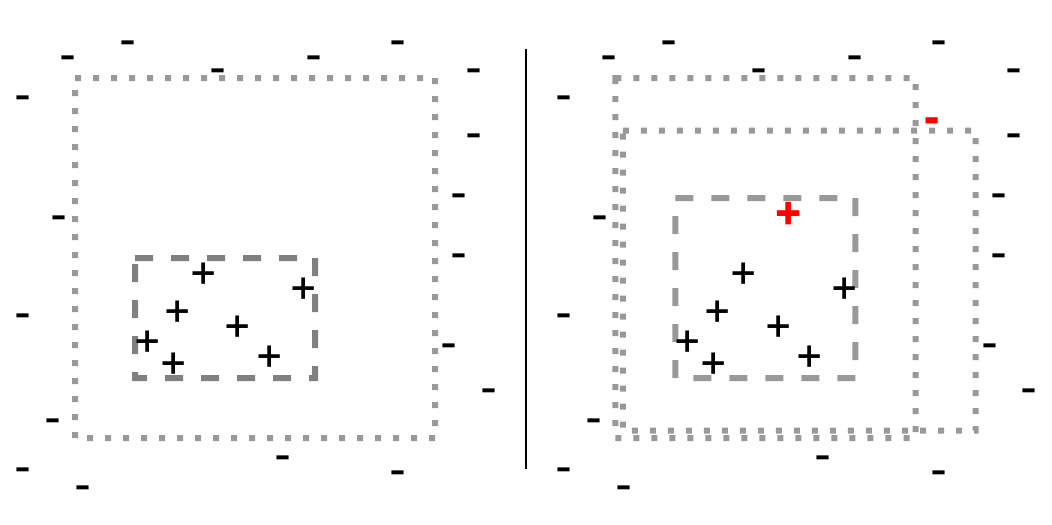}
    \caption{Version spaces and version space learning.}
    \label{fig:version_space_learning}
\end{figure}

Figure \ref{fig:version_space_learning} illustrates the main ideas of version space learning. The left side shows a version space represented by its $\L$ (dashed line) and $\U$ (dotted line) boundaries. $\L$ represents the most pessimistic hypothesis, a minimal frontier that encloses only the observed positive examples ("+" signs), i.e., those of the target class. On the other hand, $\U$ represents the most optimistic hypothesis, a maximal frontier that keeps all negative examples ("-" signs) outside. The $\L$ boundary guarantees that any unseen example within its region will be positive, whilst the $\U$ boundary guarantees that everything outside its frontier will be negative.
% The $\L$ boundary represents the most specific hypothesis in the version space while the $\U$ boundary
% If a new example appears in the region between these two boundaries, it is unknown whether it will be positive or negative.
The version space learning algorithm updates these boundaries as new learning examples appear. On the right hand of the figure, we show how these boundaries have been updated after two new learning examples (in red) are observed. When the new example is positive, the $\L$ boundary grows to contain it; otherwise, if it is negative, the $\U$ boundary shrinks to reject it. A boundary may consist of one or more hypothesis, e.g., in our figure the $\U$ boundary contains two hypotheses after the update.

%A hypothesis space is boundary-set representative iff it is convex and definite. A hypothesis space $\H$ is convex iff given two hypotheses $h_1$ and $h_2$ in $\H$, any element between them in the partial order is also in $\H$. A hypothesis space $\H$ is definite if all hypothesis in $\H$ are in the boundaries or lie between two boundary elements. 

% "The efficient representation of a version space by its boundaries only requires that some partial order be defined on it." \da{does any partial order ensure convexity and definiteness?}

% By applying this general recipe to the action model learning problem we are able to learn a compact representation of all action models that are consistent with the demonstrations. 
% In the following, we describe how to instantiate these requirements for action model learning.

% ================

% \da{move somewhere else} Given the $S$ and $G$ boundaries of a version space, the version space consists of all the subsets of $S$ that contain $G$. All such sets can be written as $G \cup X$ where $X$ is a subset of $S \setminus G$. Therefore, the size of the version space is $2^{S\setminus G}$.

\section{Version Space Learning for Action Models}\label{sec:mapping}

This section proposes a novel framework for action model learning based on version spaces. Roughly, our approach learns the precondition and effect of an action by computing a version space of its preconditions and a version space of its effects. We start by defining the hypothesis spaces and the update functions. Then, we present our algorithm.

\subsection{The Hypothesis Space}

% For simplicity, with abuse of notation, in our context a hypothesis is a subset of $L$, i.e., $h \in 2^L$ with the meaning that, let $s \in 2^L$ we have that $h \subseteq s $ iff $h$ is accepting $s$ otherwise it is not.  Let D = $s,a,s' \neq \bot $, $h$ is consistent with D iff $h$ is accepting $s$.

For simplicity, with abuse of notation, in our context a hypothesis $h$ is a subset of $L$ that represents a function. When $h$ is a \emph{precondition hypothesis}, $h$ represents the \emph{applicability} function $App_h : S \rightarrow \{0,1\}$ defined as $App_h(s) = h \subseteq s$. On the other hand, if $h$ is an \emph{effect hypothesis}, $h$ represents the \emph{successor} function $Suc_h : S \rightarrow S$ defined as $Suc_h(s) = (s \setminus \overline{h}) \cup h$. 

Let $a \in A$, with $\Hp = 2^L$ we define the hypothesis space of $a$'s preconditions, while with $\He = 2^L$ the hypothesis space of its effects. We order our hypothesis spaces using a set inclusion relation. Specifically, given two precondition hypotheses $h_1$ and $h_2$ in $\Hp$, $h_1 \leq h_2$ iff $h_1 \supseteq h_2$, and given two effect hypotheses $h_1$ and $h_2$ in $\He$, $h_1 \leq h_2$ iff $h_1 \subseteq h_2$. Note the opposite direction of the inclusion relation.

% $\leq \equiv \subseteq$ for $\Hp$, and $\leq \equiv \supseteq$ for $\He$. Hence, given two hypotheses $h_1$ and $h_2$ in $\Hp$, we say that $h_1$ is less than $h_2$ (resp. greater than), denoted $h_1 \leq h_2$ (resp. $h_1 \geq h_2$), iff $h_1 \supseteq h_2$ (resp. $h_1 \subseteq h_2$). 
% Similarly, given two hypotheses $h_1$ and $h_2$ in $\He$, we say that $h_1$ is less than $h_2$ (resp. greater than), denoted $h_1 \leq h_2$ (resp. $h_1 \geq h_2$), iff $h_1 \subseteq h_2$ (resp. $h_1 \supseteq h_2$). 

Considering these hypothesis spaces, the learning examples will be pairs $(s,b) \in S \times \{0,1\}$ for the preconditions, and pairs $(s,s') \in S \times S$ for the effects. Learning examples are implicitly given by the demonstrations. A positive demonstration $\tuple{s,a,s'}$ entails the learning example $(s,1)$ for the precondition and the learning example $(s,s')$ for the effect. On the other hand, a negative demonstration $\tuple{s,a,\bot}$ entails only the learning example $(s,0)$ for the precondition. Hereinafter, $\Dp$ and $\De$ denote the learning examples entailed by a set of demonstrations $D$ for $\Hp$ and $\He$, respectively.
% Let $D$ be a set of demonstrations, we denote with $D_\E$ the set of induced learning examples.
A hypothesis $h \in \Hp$ is consistent with a learning example $(s,b)$ iff $App_h(s) = b$. Analogously, a hypothesis $h \in \He$ is consistent with a learning example $(s,s')$ iff $Suc_h(s) = s'$.

% Note that $\Hp$ and $\He$ can be learnt in a separate fashion.
The following theorem shows that, when the learning examples come from the same set of demonstrations $D$, any model $M$ built using preconditions and effects from the learnt version spaces is a solution of the action model learning problem $\Lambda = \tuple{F,A,D}$, i.e., $M \in \M_D$.

% given the version spaces $\V_{\Hp,\Dp}$ and $\V_{\He,\De'}$ such that $E$ and $E'$ are the learning examples induced by a set of demonstrations $D$, any action model $M$ whose belong to . 

\begin{thm}\label{thm:decomposition}
    Let $\V_{\Hp,\Dp}$ and $\V_{\He,\De}$ be the version spaces of preconditions and effects of $a \in A$. The action model $M = \tuple{F,A,\pre,\eff}$ belongs to $\M_D$ if and only if $\forall a \in A : \pre(a) \in \V_{\Hp,\Dp} \land \eff(a) \in \V_{\He,\De'}$.
\end{thm}

\begin{proof}
%Let $M = \tuple{F,A,\pre,\eff}$ be an action model such that $\forall a \in A : \pre(a) \in \V_{\Hp,\Dp} \land \eff(a) \in \V_{\He,\De'}\}$.
Let $d = \tuple{s,a,s'}$ be a positive demonstration in $D$ that entails the learning examples $(s,1) \in \Dp$ and $(s,s') \in \De$. From the definition of version space, $(\pre(a),\eff(a))$ belongs to $\V_{\Hp,\Dp} \times \V_{\He,\De}$ if and only if $App_{\pre(a)}(s) = 1$ and $Suc_{\eff(a)}(s) = s'$ or, equivalently, iff $\pre(a) \subseteq s$ and $s' = (s \setminus \overline{\eff(a)}) \cup \eff(a)$. Now, let $d = \tuple{s,a,\bot}$ be a negative demonstration in $D$ and $(s,0) \in \Dp$ the entailed learning example. Again, $\pre(a) \in \V_{\Hp,\Dp}$ if and only if $App_{\pre(a)}(s) = 0$, i.e., iff $\pre(a) \not\subseteq s$. Therefore, if $\forall a \in A : \pre(a) \in \V_{\Hp,\Dp} \land \eff(a) \in \V_{\He,\De}$, $M$ satisfies conditions (1) and (2) of Definition \ref{def:aml_problem}, i.e., $M \in \M_D$; otherwise, $M \not\in \M_D$.
\end{proof}

\subsection{Initializing and Updating the Version Space}

The initialization of the version space learning algorithm sets the version space to contain the whole hypothesis space, i.e., $\V_{\Hp,\emptyset} = \Hp$ and $\V_{\He,\emptyset} = \He$. This is done by setting the $\L$ and $\U$ boundaries to contain the minimal and maximal elements of the hypothesis space, respectively. In our problem, for all actions $a \in A$, $\L_{\Hp,\emptyset} = \{L\}$, $\U_{\Hp,\emptyset} = \{\emptyset\}$, $\L_{\He,\emptyset} = \{\emptyset\}$, and $\U_{\He,\emptyset} = \{L\}$. 
Note that, by Theorem \ref{thm:decomposition} and the definition of version spaces, these boundaries allow to compactly represent the full space of action models $\M$.

% Moreover, we use $\L_{\Hp,E_D}$ ($\U_{\Hp,E_D}$) to denote the boundaries of version space $\Hp$ given the set of demonstrations $D$. Similarly for the effects. 

% An update function takes the boundaries of the version space, and a new learning example, and returns modified boundaries that represent the updated version space. Broadly, the update function checks that every hypotheses in the boundaries is consistent with the new example. If a hypothesis is inconsistent, we replace it with a comparable hypothesis within the version space that is consistent and does not leave out of the updated version space any consistent hypothesis. 

% In order to maintain a version space, we must ensure that its boundaries are consistent with all seen learning examples. Therefore, updating the version space consists in either (1) replacing any inconsistent hypothesis in the boundaries its closest (relative to the partial order) comparable hypothesis that is consistent with the example or (2) removing it if no replacement exists. When working with hypotheses in the form of sets, as we do here, the update consists in finding the smallest superset or the largest subset that is consistent with the learning example.

As illustrated in Figure \ref{fig:version_space_learning}, updating a version space involves extending the $\L$ boundary or shrinking the $\U$ boundary. This is done by modifying the hypothesis that constitute the boundaries or removing them. In our context, where hypotheses are sets, the update consists in finding the smallest superset or the largest subset that is consistent with the new demonstration. Next theorem shows how the boundaries for the version space of $a$'s preconditions $\V_{\Hp}$ are updated. Intuitively, we extend the $\L$ boundary by removing any literal not in pre-state of a positive demonstration, whilst we shrink the $\U$ boundary by adding some literal not in the pre-state of a negative demonstration. 

% \begin{thm}[Update rules for $\V_{\Hp}$]\label{thm:rules_precondition}
% Let $\L_{\Hp,D}$ and $\U_{\Hp,D}$ be the boundaries of a version space $\V_{\Hp,D}$ and 
% $d = \tuple{s, a, s'}$ a demonstration. The updated version space $\V_{\Hp,D'_{\E}}$, with $D' = D \cup \{d\}$, is given by the following rules.

% \noindent If $d$ is a positive demonstration:
% \begin{itemize}
% \item \textbf{RUP.} Remove inconsistent hypotheses from $\U_{\Hp}$:
% \[
% \U_{\Hp,D'} \coloneqq \{h_\U \mid h_\U \in \U_{\Hp,D} \land h_\U \subseteq s\}
% \]    
% \item \textbf{ULP.} Update hypotheses in $\L_{\Hp}$: 
% \[
% \L_{\Hp,D'} \coloneqq \{h_\L \cap s \mid h_\L \in \L_{\Hp,D}\}
% \]
% \end{itemize}

% \noindent If $d$ is a negative demonstration:
% \begin{itemize}
% \item \textbf{RLP.} Remove inconsistent hypotheses from $\L_{\Hp}$: 
% \[
% \L_{\Hp,D'} \coloneqq \{h_\L \mid h_\L \in \L_{\Hp,D} \land h_\L \not\subseteq s\}
% \]
% \item \textbf{UUP.} Update hypotheses in $\U_{\Hp}$:
% \[
% \U_{\Hp,D'} = \{h_\U \cup \{l\} \mid h_\U \in \U_{\Hp,D} \land h_\L \in \L_{\Hp,D} \land l \in h_\L \setminus s\}
% \]

% \[
% \U_{\Hp,D'} = \{h_\U \cup \{l\} \mid h_\U \in \U_{\Hp,D} \land l \in h_\L \setminus s \text{ with }  h_\L \in \L_{\Hp,D}\}
% \]
% \end{itemize}

% \end{thm}

\begin{thm}[Update rules for $\V_{\Hp}$]\label{thm:rules_precondition}
Let $\L_{\Hp,\Dp}$ and $\U_{\Hp,\Dp}$ be the boundaries of a version space $\V_{\Hp,\Dp}$ and 
$d$ a demonstration. The updated version space $\V_{\Hp,\Dp'}$, with $D' = D \cup \{d\}$, is given by the following rules.

\noindent If $d = \tuple{s,a,s'}$ is a positive demonstration:
\begin{itemize}
\item \textbf{RUP.} Remove inconsistent hypotheses from $\U_{\Hp,\Dp}$:
\[
\U_{\Hp,\Dp'} \coloneqq \{h_\U \mid h_\U \in \U_{\Hp,\Dp} \land h_\U \subseteq s\}
\]    
\item \textbf{ULP.} Update hypotheses in $\L_{\Hp,\Dp}$: 
\[
\L_{\Hp,\Dp'} \coloneqq \{h_\L \cap s \mid h_\L \in \L_{\Hp,\Dp}\}
\]
\end{itemize}

\noindent If $d = \tuple{s,a,\bot}$ is a negative demonstration:
\begin{itemize}
\item \textbf{RLP.} Remove inconsistent hypotheses from $\L_{\Hp,\Dp}$: 
\[
\L_{\Hp,\Dp'} \coloneqq \{h_\L \mid h_\L \in \L_{\Hp,\Dp} \land h_\L \not\subseteq s\}
\]
\item \textbf{UUP.} Update hypotheses in $\U_{\Hp,\Dp}$:

Let $h_\L \in \L_{\Hp,\Dp}$,
\begin{align*}
&\U_{\Hp,\Dp'} \coloneqq \{h_\U \mid h_\U \in \U_{\Hp,\Dp} \land h_\U \not\subseteq s\} \cup \\
&\cup\{h_\U \cup \{l\} \mid h_\U \in \U_{\Hp,\Dp} \land h_\U \subseteq s\ \land l \in h_\L \setminus s \}
\end{align*}

\end{itemize}

\end{thm}

\begin{proof} 
% Following the definition of version space, $\V_{\Hp,\Dp} = \{h \in \Hp \mid \exists (h_\L,h_\U) \in \L_{\Hp,\Dp} \times \U_{\Hp,\Dp} : h_\U \subseteq h \subseteq h_\L \}$. That is, the version space of preconditions contains any hypothesis that is a subset of a lower bound and a superset of an upper bound.
When $d = (s,a,s')$ is a positive demonstration, a hypothesis $h$ is inconsistent iff $h \not\subseteq s$. Rule RUP removes an upper bound $h_\U \in \U_{\Hp,\Dp}$ if $h_\U \not\subseteq s$ and, in doing so, removes any hypothesis $h$ such that $h_\U \subset h$ from the version space. Observe that, since $h_\U \not\subseteq s$, any superset of $h_\U$ will also be inconsistent. Rule ULP raises the lower bound from $h_\L$ to $h_\L \cap s$. Indeed, $h_\L \cap s \subseteq s$ and all subsets of $h_\L \cap s$ are also consistent. Note that, if $h_\L$ already satisfied $h_\L \subseteq s$, this rule causes no change, i.e., $h_\L = h_\L \cap s$; otherwise, $h_\L \cap s$ is the largest subset that is consistent since $\forall l \in h_\L \setminus (h_\L \cap s)$ it holds that $l \not\in s$ so all other subsets $h$ such that $h_\L \cap s \subset h \subseteq h_\L$ are inconsistent.

In the case that $d = \tuple{s,a,\bot}$, a hypothesis $h$ is inconsistent if $h \subseteq s$. Rule RLP removes the lower bound $h_\L \in \L_{\Hp,D}$ if $h_\L \subseteq s$ which also removes all its subsets from the version space. Indeed, for any subset $h$ of $h_\L$ it holds that $h \subseteq h_\L \subseteq s$ and, therefore, $h$ is inconsistent. Rule UUP shrinks an upper bound $h_\U \in \U_{\Hp,D}$ to the set $\{h_\U \cup \{l\} \mid l \in h_\L\setminus s\}$ if $h_\U \subseteq s$. Note that $\forall l \in h_\L\setminus s : h_\U \cup \{l\} \not\subseteq s$ so all the new upper bounds are the smallest supersets of $h_\U$ that are consistent.  
\end{proof}

Note that, after any examples, $\L_{\Hp}$ will at most contain a single hypothesis. Indeed, ULP only modifies the existing hypothesis without adding new ones. This result coincides with the well-known fact that, when learning pure conjunctive formulas (a precondition $\pre(a)$ is equivalent to a conjunction $\bigwedge_{l \in \pre(a)} l$), the $\L$ boundary will at most contain a single hypothesis, while the $\U$ boundary can grow to contain multiple hypotheses \cite{mitchell:1982:generalization}.

Next, we move on to the learning of effects. Before presenting the update rules, we introduce the following lemma that gives us a better handle on the version space of effects as it enables reasoning about consistency and inconsistency in terms of set inclusion.

\begin{lemma} \label{lem:effect_consistency} 
Given two states $s$ and $s'$, and an effect hypothesis $h \subseteq L$, $s' \setminus s \subseteq h \subseteq s'$ if and only if $s' = (s \setminus \overline{h}) \cup h$.
\end{lemma}
\begin{proof}[Proof Sketch (Full proof in appendix)]
    We can prove this by contradiction, leveraging algebra of sets and bitwise operations over a bit vector interpretation of states.
\end{proof}

The update rules for the version space of effects, presented in the next theorem, leverage Lemma \ref{lem:effect_consistency}. The intuition for these rules is that, whenever we get a new positive demonstration $\tuple{s,a,s'}$ we update the upper bound to be a subset of $s'$ and the lower bound to be a superset of $s' \setminus s$. By Lemma \ref{lem:effect_consistency}, any hypothesis between the updated bounds will also be consistent.

% \begin{thm}[Update rules for $\V_{\He}$]\label{thm:rules_effects}
% Let $\L_{\He,D}$ and $\U_{\He,D}$ be the boundaries of a version space $\V_{\He,D}$ and $d = \tuple{s, a, s'}$ a positive demonstration. The updated version space $\V_{\He,D'}$, with $D' = D \cup \{d\}$, is given by the following rules:
% \begin{itemize}
% \item \textbf{RLE.} Remove inconsistent hypotheses from $\L_{\He,D}$
% \[
% \L_{\He,D'} = \{h_\L \mid h_\L \in \L_{\He,D} \land h_\L \subseteq s'\}
% \]

% \item \textbf{ULE.} Update hypotheses in $\L_{\He,D}$: 
% \[
% \L_{\He,D'} \coloneqq \{h_\L \cup (s'\setminus s) \mid h_\L \in \L_{\He,D}\}
% \]

% \item \textbf{RUE.} Remove inconsistent hypotheses from $\U_{\He,D}$
% \[
% \U_{\He,D'} = \{h_\U \mid h_\U \in \U_{\He,D} \land s' \setminus s \subseteq h_\U\}
% \]

% \item \textbf{UUE.} Update hypotheses in $\U_{\He,D}$:
% \[
% \U_{\He,D'} \coloneqq \{h_\U \cap s' \mid h_U \in \U_{\He,D}\}
% \]
% \end{itemize}   
% \end{thm}

\begin{thm}[Update rules for $\V_{\He}$]\label{thm:rules_effects}
Let $\L_{\He,\De}$ and $\U_{\He,\De}$ be the boundaries of a version space $\V_{\He,\De}$ and $d = \tuple{s,a,s'}$ a positive demonstration. The updated version space $\V_{\He,\De'}$, with $D' = D \cup \{d\}$, is given by the following rules:
\begin{itemize}
\item \textbf{RLE.} Remove inconsistent hypotheses from $\L_{\He,\De}$
\[
\L_{\He,\De'} \coloneqq \{h_\L \mid h_\L \in \L_{\He,\De} \land h_\L \subseteq s'\}
\]

\item \textbf{ULE.} Update hypotheses in $\L_{\He,\De}$: 
\[
\L_{\He,\De'} \coloneqq \{h_\L \cup (s'\setminus s) \mid h_\L \in \L_{\He,\De}\}
\]

\item \textbf{RUE.} Remove inconsistent hypotheses from $\U_{\He,\De}$
\[
\U_{\He,\De'} \coloneqq \{h_\U \mid h_\U \in \U_{\He,\De} \land s' \setminus s \subseteq h_\U\}
\]

\item \textbf{UUE.} Update hypotheses in $\U_{\He,\De}$:
\[
\U_{\He,\De'} \coloneqq \{h_\U \cap s' \mid h_U \in \U_{\He,\De}\}
\]
\end{itemize}   
\end{thm}

\begin{proof}
% Following the definition of version space, $\V_{\He,\De} = \{h \in \He \mid \exists (h_\L,h_\U) \in \L_{\He,\De} \times \U_{\He,\De} : h_\L \subseteq h \subseteq h_\U \}$. That is, the version space of effects contains any hypothesis that is a superset of a lower bound and a subset of an upper bound.
Rule RLE removes the lower bound $h_\L \in \L_{\He,\De}$ when $h_\L \not\subseteq s'$ and, in doing so, removes all its supersets from the version space. Note that, if $h_\L \not\subseteq s'$, all supersets of $h_\L$ will also be inconsistent.
Rule ULE raises the lower bound from $h_\L$ to $h_\L \cup (s' \setminus s)$ which is a superset of $(s' \setminus s)$ and, therefore, consistent. If $(s' \setminus s) \subseteq h_\L$ this rule produces no change; otherwise, ULE computes the smallest superset of $h_\L$ that is consistent. Indeed, $\forall l \in (h_\L \cup (s'\setminus s)) \setminus h_\L : l \in s' \setminus s$ so removing any newly added literal from the hypothesis would make it inconsistent.

% for any other superset $h$ such that $h_\L \subseteq h \subset h_\L \cup (s' \setminus s)$ it holds that $s' \setminus s \not\subseteq h$.

% Note that, $\forall l \in h'_\L \setminus h_\L : s' \setminus s \not\subseteq h'_\L \setminus \{l\}$ and, by (2), all hypotheses inside the removed chains are inconsistent. On the other hand, the new lower bound $h'_\L$ satisfies that $s' \setminus s \subseteq h'_\L$ so the hypotheses in $chain(h'_\L, h_\U)$ are consistent by (3).

Rule RUE removes the upper bound $h_\U \in \U_{\He,\De}$ when $s' \setminus s \not\subseteq h_\U$ which also removes all its subsets from the version space. Indeed, if $s' \setminus s \not\subseteq h_\U$, then no subset of $h_\U$ can be consistent.
Rule UUE lowers the upper bound from $h_\U$ to $h_\U \cap s'$ which is consistent since $h_\U \cap s' \subseteq s'$. If $h_\U$ was already a subset of $s'$, this rule produces no change; otherwise, $h_\U \cap s' \subseteq s'$ is the largest subset of $h_\U$ that is consistent.
\end{proof}

%This section lays out some insights about action model learning that emerge from the theory we have developed.
The following corollaries, derived by theorems \ref{thm:rules_precondition} and \ref{thm:rules_effects}, interpret the update rules in an offline fashion, after any number of demonstrations.

\begin{corollary}\label{cor:one}
    $\L_{\Hp,\Dp} = \{hp_\L\}$ s.t. $hp_\L = \bigcap_{(s,1)\in \Dp} s$.
\end{corollary}

\begin{corollary}\label{cor:two}
    $\L_{\He,\De} = \{he_\L\}$ s.t. $he_\L = \bigcup_{(s,s')\in \De} s' \setminus s$
\end{corollary}

\begin{corollary}\label{cor:third}
    $\U_{\He,\De} = \{he_\U\}$ s.t. $he_\U = \bigcap_{(s,s')\in \De} s'$
\end{corollary}

Corollaries \ref{cor:one} and \ref{cor:third} show that learning a consistent action model is as easy as intersecting all pre-states for the preconditions and all post-states for the effects. In addition, Corollary \ref{cor:two} states that tighter effects can be obtained by joining all pre-state to post-state deltas. Note that, $he_\L$ can also be computed using $hp_\L$ and $he_\U$. Formally:

% These shows thatmeans that using the intersection of all pre-states as precondition and the intersection of all post-states as effect yields a consistent model that solves the problem. The next theorem shows that the difference between two consistent effect hypotheses is contained by the most restrictive precondition hypothesis.

\begin{lemma}\label{lem:optimization}
    Let $\L_{\Hp,\Dp} = \{hp_\L\}$, $\L_{\He,\De} = \{he_\L\}$ and $\U_{\He,\De} = \{he_\U\}$, we have that $he_\L = he_\U \setminus hp_\L$.
\end{lemma}

\begin{proof}[Proof sketch (Full proof in appendix)]
    This can be proven by applying De Morgan's Law and using Lemma \ref{lem:effect_consistency} to simplify equations.
\end{proof}

Thanks to Lemma \ref{lem:optimization}, we do not need to maintain the $\L_{\He}$ boundary. Moreover, we leverage it in the following lemma, which implies that all consistent models using the lower bound $hp_\L$ for their preconditions induce the same transition system. We will use this result to prove Theorem \ref{thm:soundness}.

\begin{lemma}\label{lem:indirect_result}
Let $\L_{\Hp,\Dp} = \{hp_\L\}$ and $\V_{\He,\De}$. The following holds $\forall he,he' \in \V_{\He,\De} : he \setminus he' \subseteq hp_\L$. 
\end{lemma}

\begin{proof}[Proof sketch (Full proof in appendix)]
We can leverage Lemma \ref{lem:optimization} and the inclusion relation between effect hypothesis to prove this lemma.
\end{proof}

\subsection{The VSLAM Algorithm}

\begin{algorithm}[t]
\caption{VSLAM}\label{alg:general}
\hspace*{\algorithmicindent} \textbf{Input} Action Model Learning problem $\tuple{F,A,D}$\\
\hspace*{\algorithmicindent} \textbf{Output} $\L_{\Hp}$, $\U_{\Hp}$, $\L_{\He}$ and $\U_{\He}$ for all $a \in A$

\begin{algorithmic}[1]

\For{$a \in A$}\Comment{Inizialisation}
    \State $\L_{\Hp} \coloneqq \{L\}$
    \State $\U_{\Hp} \coloneqq \{\emptyset\}$
    \State $\L_{\He} \coloneqq \{\emptyset\}$
    \State $\U_{\He} \coloneqq \{L\}$
\EndFor

\For{$\tuple{s,a,s'} \in D$}\Comment{Online loop}
\If{$s'$ is not $\bot$}
    \State $\U_{\Hp} \coloneqq RUP(\U_{\Hp}, (s,1))$
    \State $\L_{\Hp} \coloneqq ULP(\L_{\Hp}, (s,1))$
    \State $\L_{\He} \coloneqq RLE(\L_{\He}, (s,s'))$
    \State $\L_{\He} \coloneqq ULE(\L_{\He}, (s,s'))$
    \State $\U_{\He} \coloneqq RUE(\U_{\He}, (s,s'))$
    \State $\U_{\He} \coloneqq UUE(\U_{\He}, (s,s'))$
\Else
    \State $\L_{\Hp} \coloneqq RLP(\L_{\Hp}, (s,0))$
    \State $\U_{\Hp} \coloneqq UUP(\U_{\Hp}, (s,0))$
\EndIf
\EndFor

\Return $(\L_{\Hp},\U_{\Hp},\L_{\He},\U_{\He})$

\end{algorithmic}
\end{algorithm}

In this section we present VSLAM, our algorithm for action model learning, outlined in Algorithm \ref{alg:general}. VSLAM takes as input an action model learning problem $\tuple{F,A,D}$ and returns the boundaries of the version space of preconditions and of effects for each action in $A$. 

The pseudocode for VSLAM is a straightforward instantion of the initialization and update rules presented in the previous section. First, from lines 1 to 5, VSLAM initializes the version spaces associated to each action. Then, from lines 6 to 16, VSLAM processes all demonstrations in $D$ in an online fashion and uses the induced learning examples to update the boundaries of the version spaces by applying the updates rules presented in theorems \ref{thm:rules_precondition} and \ref{thm:rules_effects}.

VSLAM has some important properties that are derived from its version spaces foundation. First, by Theorem \ref{thm:decomposition}, the boundaries computed by VSLAM capture exactly all models in $\M_D$, i.e., all the solutions of $\tuple{F,A,D}$. Second, VSLAM can detect when it has learnt the true model by checking convergence of the version space. A version space \emph{converges} when only one consistent hypothesis remains, i.e., when both boundaries are singletons and contain the same hypothesis. Under our working assumption that the true model $\A$ follows the syntax of Definition \ref{def:action_model}, theorems \ref{thm:rules_precondition} and \ref{thm:rules_effects}, ensure that no consistent hypothesis is discarded and, therefore, it follows that if only one hypothesis remains it must match the true model $\A$.

% which, under our working assumptions, necessarily implies that the remaining hypothesis matches the true model. This is detected when both boundaries are singletons and contain the same hypothesis.

\begin{corollary}[Convergence]
    Let $\A = \tuple{F,A,\pre,\eff}$ be the true model. If $\L_{\Hp,\Dp} = \U_{\Hp,\Dp} = \{hp\}$, then $hp = \pre(a)$, and if $\L_{\He,\De} = \U_{\He,\De} = \{he\}$, then $he = \eff(a)$.
\end{corollary}

Lastly, VSLAM can also detect if our working assumptions are violated by checking for collapse of the version space. A version space \emph{collapses} when it becomes empty, i.e., when no consistent hypothesis remains. This indicates that the learning examples are noisy or that the hypothesis space does not contain the true model.

% Both of these situations violate our assumptions but, if they were to happen, VSLAM would be able to detect the violation. 

% In our algorithm, a collapse may happen after the application of rules RUP, RLP, RLE, and RUE. In particular, if RLP, RLE, or RUE remove any inconsistent hypothesis, the version space is certain to collapse since the boundaries $\L_{\Hp}$, $\L_{\He}$ and $\U_{\He}$ will become empty. 

\section{Sound and Complete Action Models}\label{sec:sound_complete}

In the previous section we have shown how to compute all solutions of an action model learning problem. Now, we put the focus on learnt models that guarantee some formal property of interest. In particular, we consider \emph{soundness} \footnote{Soundness has previously been referred to as "safeness" \cite{Stern:2017:SAM}} and \emph{completeness} and we show how to manipulate the computed version spaces to build sound and complete action models.

% We will also show that this compact representation can be manipulated to extract sound and complete models.

% The soundness property has been extensively studied by a stream of works published in recent years (SAM citations) but discussing it alongside completeness, the main focus of this work, serves to present a comprehensive view of the topic of learning with guarantees and to contextualize our work.

\subsection{Soundness and Completeness for Action Models}
\label{sec:soundness_completeness}

A model $M$ is \emph{sound} with respect to another model $M'$ if every transition of $M$ is also a transition of $M'$. In contrast, a model $M$ is \emph{complete} with respect to another model $M'$ if every transition of $M'$ is also a transition of $M$.

\begin{definition}[Soundness]\label{def:soundness}
    Let $M$ and $M'$ be two action models, we say that $M$ is \emph{sound} with respect to $M'$ iff $\T_M \subseteq \T_{M'}$.
\end{definition}

\begin{definition}[Completeness]\label{def:completeness}
    Let $M$ and $M'$ be two action models, we say that $M$ is \emph{complete} with respect to $M'$ iff $\T_M \supseteq \T_{M'}$.
\end{definition}

By looking at definitions \ref{def:soundness} and \ref{def:completeness}, it is quite obvious that soundness and completeness are opposite properties and this relationship also carries on to solution plans computed with such models. It is easy to see that, if $M$ is sound w.r.t. $M'$ then any solution plan for a planning problem $P = \tuple{M,s_0,G}$ will also be a solution plan for $P' = \tuple{M',s_0,G}$, i.e., $\Pi(P) \subseteq \Pi(P')$. On the other hand, if $M$ is complete with respect to $M'$, then all solution plans for $P' = \tuple{M',s_0,G}$ will also be valid solutions for $P = \tuple{M,s_0,G}$, i.e., $\Pi(P) \supseteq\Pi(P')$. 

% In the context of action model learning, we use the term \emph{sound action model} to refer to a learnt model that is sound with respect to the true model. Similarly, for \emph{complete action model}. Sound action models are desirable because they allow operating with an agent even when we do not have a complete understanding of its capabilities. This means that, if we are able to learn an action model $M$ that is sound with respect to the true model $\A$, we will be able to compute plans with $M$ that can be carried out by $\A$. Complete models have clear utility when it comes to solvability questions. After all, if a planning problem has no solution under the complete model it will also be unsolvable under the true model. 

In the context of action model learning, we use the term \emph{sound action model} to refer to a learnt model that is sound with respect to the true model $\A$. Similarly, for \emph{complete action model}. 
% Sound action models are desirable because they allow operating with an agent even when we do not have a complete understanding of its capabilities.
Next, we define \emph{sound action model learning} and \emph{complete action model learning}, two specializations of action model learning that consider, respectively, only sound or complete models as solutions. 
% To avoid confusion, for the remainder of this section, we refer to the set of solutions of the vanilla action model learning problem $\Lambda = \tuple{F,A,D}$, i.e., $\M_D$, as \emph{consistent models}.
%TODO: motivate it

\begin{definition}[Sound Action Model Learning Problem]
A \emph{sound action model learning problem} is a tuple $\Lambda_S = \tuple{F, A, D}$ where $F$ is a set of fluents, $A$ is a set of actions, and $D$ is a set of demonstrations. A solution for $\Lambda_S$ is a model $M$ that is sound with respect to all models in $\M_D$.
\end{definition}

Our definition of solution for $\Lambda_S$ relies on the observation that, while we do not know the true model $\A$, we know that $\A$ must belong to the set of models consistent with $D$. Therefore, to ensure that our solution model is sound with respect to $\A$, it must be sound with respect to all models in $\M_D$. The quality of the solution is critical in this problem, since there exist models that are trivially sound. For example, a model $M = \tuple{F,A,\pre,\eff}$ with $\pre(a) = L$ for all $a \in A$ is trivially sound since $\T_M = \emptyset$ but such a model is obviously of little interest. For this reason, given two models $M$ and $M'$ that solve a sound action model learning problem $\Lambda_S$, we say that $M$ is a better solution than $M'$ if $\T_M \supset \T_{M'}$, and say that $M$ is an optimal solution for $\Lambda_S$ if there exists no other solution $M'$ such that $\T_M \subset \T_{M'}$.

\begin{definition}[Complete Action Model Learning Problem]
A \emph{complete action model learning problem} is a tuple $\Lambda_C = \tuple{F, A, D}$ where $F$ is a set of fluents, $A$ is a set of actions, and $D$ is a set of demonstrations. A solution for $\Lambda_C$ is a model $M$ that is complete with respect to all models in $\M_D$.
\end{definition}

As before, to guarantee that the solution will be complete with respect to $\A$, it must be complete with respect to any model consistent with $D$. Regarding solution quality, we say that, given two models $M$ and $M'$ that solve a complete action model learning problem $\Lambda_C$, $M$ is a better solution than $M'$ if $\T_M \supset \T_{M'}$, and $M$ is an optimal solution for $\Lambda_S$ if there exists no other solution $M'$ such that $\T_M \subset \T_{M'}$.

\subsection{Extracting Sound Models from a Version Space}

% To solve a sound action model learning problem $\Lambda_S = \tuple{F,A,D}$, we need to learn a model $M$ whose preconditions and effects induce a transition system that is a subset of that induced by the true model $\A$.

In this section we show that the version spaces computed by VSLAM, and more precisely their lower boundaries $\L_{\Hp}$ and $\L_{\He}$, can be used to build sound models. Further, such sound models are optimal solutions.

\begin{thm}\label{thm:soundness}
    Let $\Lambda_S = \tuple{F,A,D}$ be a sound action model learning problem. The action model $M = \tuple{F,A,\pre,\eff}$ such that $\pre(a) \in \L_{\Hp,\Dp}$ and $\eff(a) \in \L_{\He,\De}$ for all $a \in A$ is an optimal solution for $\Lambda_S$.
\end{thm}

\begin{proof}
    Let $\T_{\M_D} = \bigcap_{M' \in \M_D} \T_{M'}$ denote the intersection of all transition systems induced by consistent models $\M_D$. We show that $\T_M = \T_{\M_D}$.
    
    % For this proof, observe that if the version spaces $\V_{\Hp,D}$ and $\V_{\He,D}$ did not collapse, their greatest lower bounds $\L_{\Hp,D}$ and $\L_{\He,D}$ are singletons and can be precisely defined. Indeed, $\pre(a) = \bigcap_{\tuple{s,a,s'} \in D(a)} s$ and $\eff(a)=\bigcup_{\tuple{s,a,s'} \in D(a)} s' \setminus s$ where $D(a)$ denotes the subset of transitions in $D$ with $a$ as the action.

    (Soundness) First, we proof that $\T_M \subseteq \T_{\M_D}$. By contradiction, assume that there exists a transition $t = \tuple{s,a,s'}$ such that $t \in \T_M$ but $t \not\in \T_{\M_D}$. For $t \not\in \T_{\M_D}$ to be true, there must exist a model $M' = \tuple{F,A,\pre',\eff'}$ in $\M_D$ such that either (1) $a$ is not applicable in $s$, or (2) its execution results in a state $s'' \neq s'$. 
    By construction, it holds that $\pre(a)' \subseteq \pre(a)$. Consequently, if $\pre(a) \subseteq s$ then $\pre(a)' \subseteq s$ and case (1) does not hold. 
    For case (2), we know from Lemma $\ref{lem:indirect_result}$ that $\eff'(a) \setminus \eff(a) \subseteq \pre(a)$ so it follows that $\eff'(a) \setminus \eff(a) \subseteq s$. Meaning that, any difference between $\eff(a)$ and $\eff'(a)$ is already part of state $s$ and the execution of $a$ cannot result in different states.

    (Optimality) We prove $\T_M \supseteq \T_{\M_D}$. For this, simply observe that, by Theorem \ref{thm:decomposition}, $M \in \M_D$. Then, $\T_{\M_D}$ is a subset of any of its intersecting sets including $\T_M$.
    Finally, since $\T_M \subseteq \T_{\M_D}$ and $\T_M \supseteq \T_{\M_D}$, it follows that $\T_M = \T_{\M_D}$ and, therefore, $M$ is an optimal solution for $\Lambda_S$.
\end{proof}

\subsection{Extracting Complete Models from a Version Space}

A complete action model must be able to produce any of the transitions generated by any of the consistent models in $\M_D$. Intuitively, the actions of a complete model should be applicable in any state where it is applicable according to a consistent model, and its execution should generate all possible post-states generated under any consistent model. It is easy to see that working under the limits of Definition \ref{def:action_model} leads to two problems. First, a model that produces multiple possible post-states is, by definition, non-deterministic. And second, using conjunctive preconditions may easily lead to a weak precondition that accepts more pre-states than necessary.
Therefore, we target a more expressive planning model that accommodates disjunctive preconditions and non-deterministic effects. 

\begin{definition}[Non-deterministic Action Model]\label{def:nd_action_model}
    A non-deterministic action model is a tuple $M = \tuple{F,A,\Pre,\Eff}$ where:
    \begin{itemize}
        \item $F$ and $A$ are finite sets of fluents and actions as given in Definition $\ref{def:action_model}$.
        \item $\Pre : A \rightarrow 2^{2^L}$ defines the \emph{set of preconditions} $\Pre(a) \subseteq 2^L$ of each action $a \in A$.
        \item $\Eff : A \rightarrow 2^{2^L}$ defines the \emph{set of effects} $\Eff(a) \subseteq 2^L$ of each action $a \in A$.
    \end{itemize}
\end{definition}

The main difference with respect to Definition \ref{def:action_model} is that here each action is associated to a set of preconditions and a set of effects. A non-deterministic action model $M = \tuple{F,A,\Pre,\Eff}$ induces the transition system $\T_M \subseteq S \times A \times S$ consisting of all transitions $\tuple{s,a,s'}$ that satisfy $\exists p \in \Pre(a): p \subseteq s$ and $\exists e \in \Eff(a)$ such that $s' = (s \setminus \overline{e}) \cup e$. Using this new formulation, we are able to build a complete action model that is, in addition, optimal.

\begin{thm}
    Let $\Lambda_C = \tuple{F,A,D}$ be a complete action model learning problem. The non-deterministic action model $M = \tuple{F,A,\Pre,\Eff}$ s.t. $\Pre(a) = \U_{\Hp,\Dp}$ and $\Eff(a) = \V_{\He,\De}$ for all $a \in A$ is an optimal solution for $\Lambda_C$.
\end{thm}

\begin{proof}
    Let $\T_{\M_D} = \bigcup_{M' \in \M_D} \T_{M'}$ denote the union of all transition systems induced by consistent models $\M_D$.  We show that $\T_M = \T_{\M_D}$.

    (Completeness) We start by proving that $\T_M \supseteq \T_{\M_D}$. By contradiction, assume that there exists a transition $t = \tuple{s,a,s'}$ such that $t \in \T_{\M_D}$ but $t \not\in \T_M$. Then, $\M_D$ must include a model $M' = \tuple{F,A,\pre', \eff'}$ such that $\pre'(a) \subseteq s$ and $s' = (s \setminus \overline{\eff'(a)}) \cup \eff'(a)$, and either (1) $\forall p \in \Pre(a) : p \not\subseteq s$, or (2) $\forall e \in \Eff(a) : s' \neq (s \setminus \overline{e}) \cup e$. Since $\Pre(a) = \U_{\Hp,\Dp}$ and $\pre'(a) \in \V_{\Hp,\Dp}$, there must exists a $p \in \Pre(a)$ such that $p \subseteq \pre'(a)$, and if $\pre'(a) \subseteq s$ is true then so must be $p \subseteq s$. This falsifies case (1). For case (2), observe that $\eff'(a) \in \V_{\He,\De}$ and $\Eff(a) = \V_{\He,\De}$. Therefore, $\eff'(a) \in \Eff(a)$ and case (2) cannot be true. 
    % Since $\T_M \supseteq \T_{\M_D}$, we can assert that $M$ is complete. 

    (Optimality) We prove $\T_M \subseteq \T_{\M_D}$. By contradiction, assume that $t \in \T_M$ but $t \not\in \T_{\M_D}$. 
    Then, it must be true that $\exists p \in \Pre(a) : p \subseteq s$ and $\exists e \in \Eff(a) : s' = (s \setminus \overline{e}) \cup e$ and $\M_D$ cannot contain a model $M' = \tuple{F,A,\pre', \eff'}$ such that $\pre'(a) \subseteq s$ and $s' = (s \setminus \overline{\eff'(a)}) \cup \eff'(a)$. However, this cannot be true since $\Pre(a) \subseteq \V_{\Hp,D}$ and $\Eff(a) = \V_{\He,D}$, so $\M_D$ contains a model $M'$ with $\pre'(a) = p$ and $\eff'(a) = e$. Finally, since $\T_M \supseteq \T_{\M_D}$ and $\T_M \subseteq \T_{\M_D}$ are true, we have that $\T_M = \T_{\M_D}$ which proves that $M$ is indeed an optimal solution for $\Lambda_C$.
\end{proof}

\section{Experimental Evaluation}\label{sec:evaluation}

We evaluate the sound and complete models produced by VSLAM empirically on a selection of planning domains. The code and benchmarks are in the supplementary material and will be made publicly available upon acceptance.

\begin{figure}[t]
    \centering
    \includegraphics[width=0.9\linewidth]{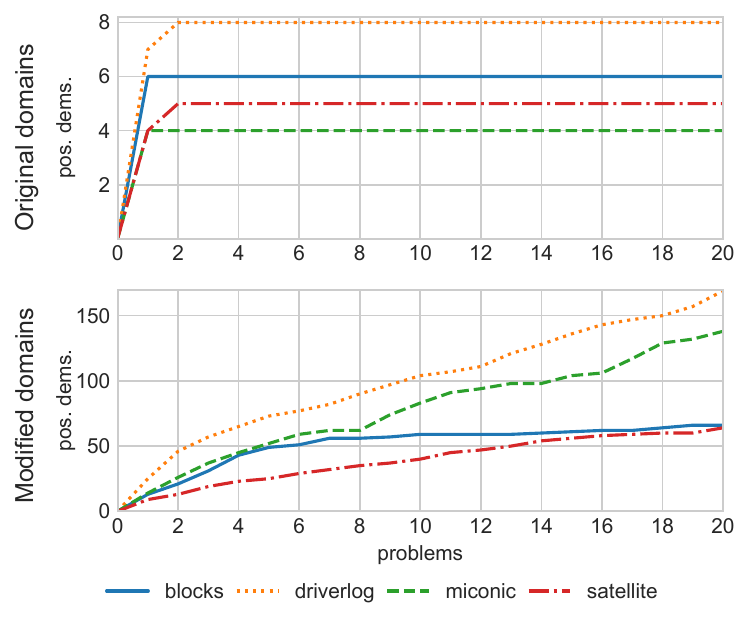}
    \caption{Distinct positive demonstrations collected with the original (top) and modified domains (bottom).}
    \label{fig:sample_analysis}
\end{figure}

\paragraph{Benchmarks.}
We select 4 domains from the International Planning Competition \cite{mcdermott:1998:IPC}, namely \block, \sate, \mic and \driver, and use publicly available generators \cite{seipp:2022:generators} to get 20 random problems for each domain. We solve all problems with LAMA \cite{richter:2010:lama} and collect each transition induced by the solution plans as a positive demonstration. Negative demonstrations are generated by randomly trying to execute actions throughout the states traversed by solution plans. 

Action model learning is usually performed over a lifted action representation. This assumption restricts the hypothesis space to those fluents that are related to the parameters of the action. To have more challenging benchmarks we modified the original actions by adding extra parameters. This modification enlarges the hypothesis space and so the number of demonstrations. As shown in Figure \ref{fig:sample_analysis}, with this simple modification, we increase by 10 to 30 times the number of \emph{distinct} demonstrations and reduce the overlap between plans. Note that with the original domains the number of positive demonstrations saturates after the second problem. 
Table \ref{tab:features} summarizes the features of our benchmarks: the second and the third columns report on the number of lifted actions and fluents and their maximal number of paratemers (in parentheses); the fourth and  the fifth columns report on the number of positive (POS) and negative (NEG) demonstrations collected.

%Commonly, in action model learning, it is assumed that the parameters of the lifted actions are known. However, in a preliminary analysis of the benchmarks, we discovered that this assumption makes the learning problem too easy. To illustrate this point, Figure \ref{fig:sample_analysis} (top) shows the accumulated positive demonstrations collected as we solve the generated problems. We observe that this number is too small (only 1 or 2 demonstrations for each action) and saturates after the second problem, indicating heavy overlap between problems. Our proposal to build more challenging benchmarks is to extend the parameters of each action. This can be interpreted as setting an upper bound on the number of objects involved in the action rather than precisely knowing them beforehand. Such a bound can come, for instance, from the maximum number of objects that an actuator can interact with. In practice, adding more parameters increases the size of the hypothesis space and the number of demonstrations, both of which lead to a more interesting learning task. As shown in Figure \ref{fig:sample_analysis} (bottom), with this simple modification, we were able to increase by 10 to 30 times the number of demonstrations and reduce the overlap between plans. Table \ref{tab:features} summarizes our benchmark domains: second and third columns describes the number of lifted actions and fluents and their maximal arity (between parentheses), fourth and fifth describes the positive (POS) and negative (NEG) demonstrations collected.

\begin{table}[ht]
    \small
    \centering
    \begin{tabular}{l|c|c|c|c}
    \textbf{DOMAIN} & $A$ & $F$ & \textbf{POS} & \textbf{NEG} \\ \hline
    \block & 4 (3) & 5 (2) & 66 & 882 \\
    \driver & 6 (8) & 6 (2) & 169 & 903 \\
    \mic & 4 (4) & 8 (2) & 138 & 381 \\
    \sate & 5 (8) & 8 (2) & 64 & 166 \\
    \end{tabular}
    \caption{Domains, features and collected demonstrations.}
    \label{tab:features}
\end{table}

\begin{figure*}[ht]
    \centering
    \includegraphics[width=\linewidth]{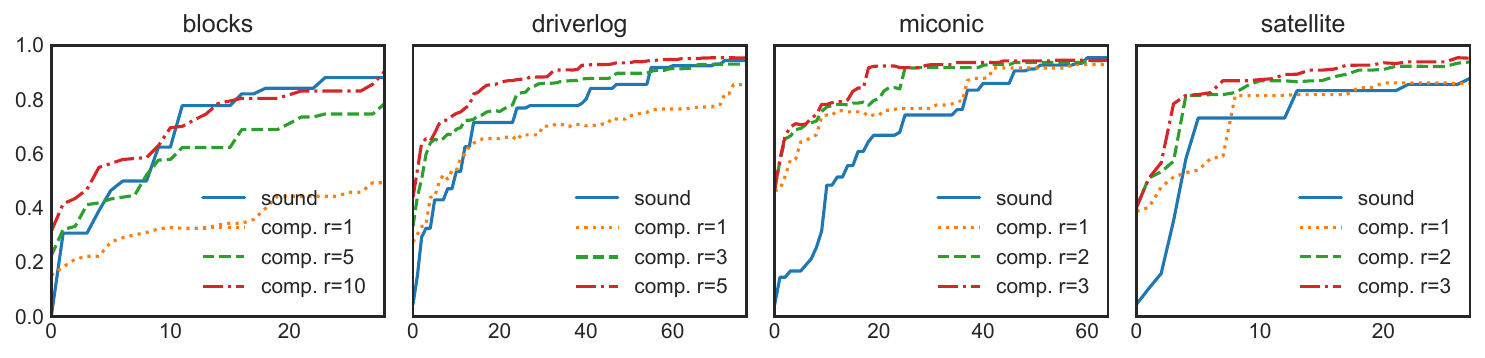}
    \caption{F1-score (y-axis) of the sound and complete action models as the training demonstrations increase (x-axis), with different positive vs negative ratios}
    \label{fig:f1score}
\end{figure*}

\paragraph{Evaluation criteria.} Following previous works on action model learning \cite{Lamanna:2021:olam,aineto:2019:fama}, we use the \emph{f1-score}, the harmonic mean of \emph{precision} and \emph{recall}. We interpret these metrics over a test set of demonstrations by having the learnt action model label them as positive or negative with respect to its transition system. \emph{Precision} degrades with \emph{false positives}, i.e., when the model accepts a negative demonstration as part of its transition system. Conversely, \emph{recall} degrades with \emph{false negatives}, i.e., when the model rejects a positive demonstration. Indeed, a sound action model will always have perfect \emph{precision} but lower \emph{recall}, meaning that it accepts few transitions but all of them are correct. In contrast, a complete model will have perfect recall yet low precision, since it will often accept transitions that do not belong to the true model. The \emph{f1-score} provides us with a very good proxy for understanding the usability of our models.
We split the collected demonstrations, using half for learning and half for testing, and measure the f1-score of both models as more demonstrations are processed. The complete model is evaluated simulating scenarios where negative demonstrations are seen at different rates. We do this by using a ratio $r$ of negative to positive demonstrations. For instance, a ratio $r=2$ indicates that the demonstrations set $D$ contains 2 negative demonstrations for every positive one. With this setting, we aim at understanding the impact of the distribution of our dataset.

\paragraph{Results.} Figure \ref{fig:f1score} illustrates our results; the x-axis reports on the number of growing positive demonstrations. For each such a positive demonstration there is a number of negative ones depending on ratio $r$; we represent such an information with different curves.
We observe that the sound action model performs better in more imbalanced domains (in terms of positive and negative demonstrations) such as \block. Instead, the complete model seems to be effective over more balanced domains like \mic and \sate. This comes with no surprise since more imbalanced distributions correlate to stricter preconditions and such models are closer (in the hypothesis space) to the sound model. The opposite holds true for the complete model. Generally, the complete model has the advantage in the earlier stages of the learning process, but in our experiment is later outperformed by the sound model after more demonstrations have been processed. Overall, no model dominates the other and their relative performance depends on the characteristics of the domain, how far into the learning we are and how accessible the positive and negative demonstrations are. This indicates a strong complementarity of the two models.

% The most important result, in our view, is that no model dominates the other, specially in the more balanced configurations. This indicates that leveraging both models will lead to the best results in attempting to solve any planning problem. 

\section{Related Work}\label{sec:related}

Research on action model learning has produced a wide variety of sophisticated learning approaches -- different surveys can be found in \citet{jimenez:2012:review,arora:2018:review,aineto:2022:review}. Starting with the pioneering works of ARMS \cite{yang:2007:arms} and SLAF \cite{amir:2008:slaf}, research in this field has been quite prolific. Approaches like LAMP \cite{zhuo:2010:lamp} investigated the learning of more expressive action models, while others like FAMA \cite{aineto:2019:fama} and AMAN \cite{zhuo:2013:aman} focused on learning from incomplete or noisy demonstrations. We can even find approaches that actively seek the demonstrations that will help them learn faster \cite{Lamanna:2021:olam, verma:2021:asking}.

Broadly, most learning approaches can be classified, according to their notion of solution, into those that compute a model consistent with the demonstrations \cite{cresswell:2013:locm,aineto:2019:fama,bonet:2020:learning}, or those that target the action model that maximizes some objective or fitness function \cite{yang:2007:arms,kuvcera:2018:louga,zhuo:2010:lamp,zhuo:2013:aman,mourao:2012:alice}. While we follow the former interpretation, our approach is one of the few, alongside SLAF \cite{amir:2008:slaf}, which is able to compute all solutions to the problem. SLAF computes a CNF formula representing all possible transitions that can be regarded as a form of version space. However, this formula does not offer the compactness of our boundaries nor can be easily manipulated and, indeed, the only way to extract from it a concrete solution model is using a SAT solver. Unlike us, SLAF handles partial observability, a feature that we expect to support in the future following similar extensions for version spaces.

% and we can find works whose interpretation of the problem does not fit ours. We will discuss these other interpretations of action model learning in the related work. \da{ARMS and LOCM?}

The approach that we regard as the closest to our own is SAM \cite{Stern:2017:SAM} for its focus on safe (sound) models, learning setting and, interestingly, a foundational connection. The authors of SAM link their approach to Valiant's elimination algorithm \cite{valiant:1984:theory} and Mitchel \shortcite{mitchell:1982:generalization} himself describes this algorithm as the subproblem of computing the $\L$ boundary in version space learning. Our work refreshes this connection in an action model learning setting. Indeed, SAM focuses on one extreme of the spectrum of solution models, those guaranteeing soundness, which are tied to the $\L$ boundary. On the other extreme we find the complete models that we highlight in this work. Whether there are other interesting models in the middle between these two extremes is object of future work.

\section{Conclusions and Future Work}\label{sec:discussion}

In this paper we have proposed an approach to learning action models from first principles. We do so by exploiting version spaces to a great extent. One of the main benefits of our approach is the ability to learn in an integrated and comprehensive way sound models as for  \citet{Stern:2017:SAM} together with complete models, providing therefore a great deal of flexibility. Indeed, our framework enables an agent not only to learn from positive demonstrations but also from failures. Empirically, we observed that with this facility in place, an agent can start learning something useful for reasoning already with a few number of examples. 

There are a number of future works in our research agenda. First, we aim at lifting, or at least relaxing, some of our assumptions. In this regard, we want to investigate \emph{version space algebra} \cite{lau:2003:algebra} as a means to tackle more expressive action models. Indeed, \emph{version space algebra} allows to operate over much more complex hypothesis spaces, providing a solid foundation to target action models involving, e.g., numeric state variables \cite{fox:03:pddl}; of importance in this context is the relationship with what was done by \citet{Mordoch:2023:NSAM}. Second, we plan to support noisy demonstrations following the steps taken by \citet{Norton:1992:noise} for version spaces. Last but not least, we would like to investigate the role of version spaces in an active-learning approach. Speculatively, the region between the boundaries of the version space is where uncertainty resides so focusing our attention in this area should accelerate convergence towards the true model.

% focusing the attention only on the space between the boundaries of the version spaces should, potentially, accelerate convergence towards the true model as this is the area where uncertainty resides.

% since it is where uncertainty resides
% we could use the boundaries of the version space to focus attention only on subsets of transitions, therefore, potentially, accelerating convergence towards the true model.

\bigskip

\bibliography{aaai24}

\end{document}

\maketitle

\section{Technical Proofs}

This appendix contains the full proofs of all the lemmas in our paper.

\begin{manualtheorem}{3}\label{lem:effect_consistency} 
Given two states $s$ and $s'$, and an effect hypothesis $h \subseteq L$, $s' \setminus s \subseteq h \subseteq s'$ if and only if $s' = (s \setminus \overline{h}) \cup h$.
\end{manualtheorem} 

\begin{proof}
Starting with $s' \setminus s \subseteq h \subseteq s' \implies s' = (s \setminus \overline{h}) \cup h$. By contradiction, we assume the antecedent is true while the consequent is false. If $s' \neq (s \setminus \overline{h}) \cup h$ then either (1) $s' \not\subseteq (s \setminus \overline{h}) \cup h$ or (2) $s' \not\supseteq (s \setminus \overline{h}) \cup h$.

For (1), it means that

\[
s' \setminus ((s \setminus \overline{h}) \cup h) \neq \emptyset 
\]
\[
(s' \setminus (s \setminus \overline{h})) \cap ( s' \setminus h) \neq \emptyset 
\]
\[
((s' \cap \overline{h}) \cup (s' \setminus s)) \cap (s' \setminus h) \neq \emptyset 
\]

Since $h \subseteq s'$ then $s' \cap \overline{h} = \emptyset$

\[
(s' \setminus s) \cap (s' \setminus h) \neq \emptyset 
\]

Since $s' \setminus s \subseteq h$, it follows that $(s' \setminus s) \cap (s' \setminus h) = \emptyset$ and we arrive at a contradiction $\emptyset \neq \emptyset$.

For (2), it means that

\[
((s \setminus \overline{h}) \cup h) \setminus s' \neq \emptyset 
\]
\[
((s \setminus \overline{h}) \setminus s')  \cup (h \setminus s') \neq \emptyset 
\]

Since $h \subseteq s'$ then $h \setminus s' = \emptyset$

\[
(s \setminus \overline{h}) \setminus s' \neq \emptyset 
\]

% \[
% (s \cup s') \setminus (\overline{h} \cup s') \neq \emptyset 
% \]

\[
(s \setminus s') \setminus \overline{h} \neq \emptyset 
\]

% \[
% (s \cup \setminus \overline{h}) \setminus (s' \cup \overline{h}) \neq \emptyset 
% \]

% \[
% s \setminus (\overline{h} \cup s') \neq \emptyset 
% \]

% \da{How do you show that $s \subseteq \overline{h} \cup s'?$}

% \[
% (s \setminus \overline{h}) \cap (s \setminus s') \neq \emptyset 
% \]

Before continuing, observe that a state represents a full assignment of fluents $F$ and can be understood as $|F|$-sized bit vectors. Under this interpretation, $s - s' = \mytilde(\mytilde s + s')$ where $\mytilde$ denotes the bitwise complement (logical negation on each bit). Coming back to our set representation, this means that $\overline{s' \setminus s} = s \setminus s'$ since $s \setminus s'$ and $s' \setminus s$ contain the literals associated to fluents that are evaluated differently in $s$ and $s'$ which given the Boolean domain of fluents can only be their complementary literals.
% $s \cap s'$ contains literals of the fluents that have the same valuation in both $s$ and $s'$, while $s \setminus s'$ (or $s' \setminus s)$ contain those with a different valuation. Since fluents are Boolean variables, the difference must contain complementary literals, i.e., $\overline{s' \setminus s} = s \setminus s'$. 
% To understand this, observe that a state is representing a full assignment of fluents $F$ and can be interpreted as $|F|$-sized bit vectors.
Resuming our proof, note that, since $\overline{s' \setminus s} = s \setminus s'$ and $s' \setminus s \subseteq h$, we have that $s \setminus s' \subseteq \overline{h}$ and we arrive at a contradiction $\emptyset \neq \emptyset$

% Since $s' \setminus s \subseteq h$, we have that $s \setminus s' \subseteq \overline{h}$ \da{this seems to be true but I don't know how to prove it} and we arrive at a contradiction $\emptyset \neq \emptyset$

Now, we move on to the proof of $s' = (s \setminus \overline{h}) \cup h \implies s' \setminus s \subseteq h \subseteq s'$. Again, we assume the antecedent is true while the consequent is false. Meaning, either (3) $s' \setminus s \not\subseteq h$ or (4) $h \not\subseteq s'$.

For (3), consider that $s' = (s \setminus \overline{h}) \cup h$ implies $s' \subseteq (s \setminus \overline{h}) \cup h$.

\[
s' \subseteq (s \setminus \overline{h}) \cup h
\]
\[
s' \setminus (s \setminus \overline{h}) \subseteq h
\]
\[
(s' \setminus s) \cup (s' \cap \overline{h}) \subseteq h
\]

Hence, $(s' \setminus s) \subseteq h$ and (3) cannot be true.

Finally, (4) is very obviously false given that $s'$ is the union of $(s \setminus \overline{h})$ and $h$.

\end{proof}

\begin{manualtheorem}{8}\label{lem:optimization}
    Let $\L_{\Hp,\Dp} = \{hp_\L\}$, $\L_{\He,\De} = \{he_\L\}$ and $\U_{\He,\De} = \{he_\U\}$, we have that $he_\L = he_\U \setminus hp_\L$.
\end{manualtheorem}

\begin{proof}
For simplicity, we develop this segment of the proof assuming only two learning examples, i.e., $\De = \{(s_1, s'_1), (s_2, s'_2)\}$. 

\noindent Starting with

    $(s'_1 \cap s'_2) \setminus (s_1 \cap s_2)$

\noindent and applying De Morgan's Law we arrive at

    $((s'_1 \cap s'_2) \setminus s_1) \cup ((s'_2 \cap s'_1) \setminus s_2)$

\noindent Intersection with set difference is set difference with intersection

    $((s'_1 \setminus s_1) \cap s'_2) \cup ((s'_2 \setminus s_2 ) \cap s'_1)$

\noindent From Lemma \ref{lem:effect_consistency}, we know that for all $he \in \V_{\He,\De}$ it holds that $(s'_1 \setminus s_1) \subseteq he$, $(s'_2 \setminus s_2) \subseteq he$, $he \subseteq s'_1$ and $he \subseteq s'_2$, so it follows that $(s'_1 \setminus s_1) \subseteq s'_2$ and $(s'_2 \setminus s_2) \subseteq s'_1$ and we arrive at

    % Since $s'_i \setminus s_i \subseteq \eff(a) \subseteq s'_i$

    $(s'_1 \setminus s_1) \cup (s'_2 \setminus s_2 )$
    
Generalizing this procedure to any $\De$, we have that $he_\L = \bigcup_{(s,s')\in \De} s' \setminus s = (\bigcap_{(s,s')\in \De} s') \setminus (\bigcap_{(s,s')\in \De} s)$ and, therefore, that $he_\L = he_\U \setminus hp_\L$.

\end{proof}

\begin{manualtheorem}{9}\label{thm:indirect_result}
Let $\L_{\Hp,\Dp} = \{hp_\L\}$ and $\V_{\He,\De}$. The following holds $\forall he,he' \in \V_{\He,\De} : he \setminus he' \subseteq hp_\L$. 
\end{manualtheorem}

\begin{proof}
Let $\L_{\He,\De} = \{he_\L\}$ and $\U_{\He,\De} = \{he_\U\}$ be the boundaries of $\V_{\He,\De}$. From Lemma \ref{lem:optimization}, we know that $he_\L = he_\U \setminus hp_\L$. Since $\forall he \in \V_{\He,\De} : he_\L \subseteq he \subseteq he_\U$, it follows that $he \setminus he_\L \subseteq  he_\U \setminus he_\L$. Substituting, $he \setminus he_\L \subseteq he_\U \setminus (he_\U \setminus hp_\L)$ simplifying to $he \setminus he_\L \subseteq he_\U \cap hp_\L$.
\end{proof}